\documentclass[11pt]{article}
\usepackage[margin=1in]{geometry}
\usepackage{authblk}
\usepackage{amsmath,amssymb,amsthm,mathtools}
\usepackage{bm}
\usepackage{graphicx}
\usepackage{booktabs}
\usepackage{xcolor}
\usepackage[colorlinks=true,linkcolor=blue,citecolor=blue]{hyperref}
\usepackage{cleveref}
\crefname{assumption}{Assumption}{Assumptions}
\Crefname{assumption}{Assumption}{Assumptions}

\graphicspath{{figs/}}

\theoremstyle{plain}
\newtheorem{theorem}{Theorem}
\newtheorem{proposition}{Proposition}
\newtheorem{lemma}{Lemma}

\theoremstyle{definition}
\newtheorem{definition}{Definition}
\newtheorem{assumption}{Assumption}
\newtheorem{remark}{Remark}

\newcommand{\SE}{\ensuremath{\mathrm{SE}(3)}}
\newcommand{\so}{\mathfrak{so}(3)}
\newcommand{\se}{\mathfrak{se}(3)}
\newcommand{\R}{\mathbb{R}}
\newcommand{\E}{\mathbb{E}}
\newcommand{\norm}[1]{\left\lVert#1\right\rVert}
\newcommand{\Gnorm}[1]{\left\lVert#1\right\rVert_{G}}
\newcommand{\opnorm}[1]{\left\lVert#1\right\rVert_{\mathrm{op}}}

\newcommand{\Ginner}[2]{\langle #1,\,#2\rangle_{G}}
\newcommand{\Cv}{H}                                   % curvature matrix (bilinear form)
\newcommand{\Aop}{\ensuremath{\mathcal{A}}}           % curvature operator G^{-1}H
\newcommand{\Hoo}{\ensuremath{\mathcal{A}_{OO}}}      % restriction of A to O
\newcommand{\pinvG}{\ensuremath{\mathcal{A}^{\dagger_{G}}}} % G-pseudoinverse of A
\newcommand{\gsc}{\ensuremath{\psi^{\sharp}}}         % gradient (Riesz) score, sharp = index raised by G
\newcommand{\PO}{P_{O}}
\newcommand{\range}{\operatorname{range}}
\newcommand{\GOI}{\ensuremath{\mathrm{GOI}}}
\newcommand{\IF}{\ensuremath{\mathrm{IF}}}
\newcommand{\lmin}{\lambda_{\min}}

\newcommand{\tr}{\operatorname{tr}}

\title{\bf The Geometric Observability Index: Influence, Fisher Information,
and Weak Observability in $\SE$ Pose Estimation}

\author[1]{Joe-Mei Feng\thanks{Corresponding author: \texttt{168984@o365.tku.edu.tw}}}
\author[2]{Hsin-Hsiung Kao\thanks{\texttt{kao@mail.cpu.edu.tw}}}
\author[1]{Sheng-Wei Yu}
\affil[1]{Department of Computer Science and Information Engineering, Tamkang University, New Taipei City, Taiwan}
\affil[2]{Department of Information Management, Central Police University, Taoyuan, Taiwan}
\date{}

\begin{document}
\maketitle

\begin{abstract}
We introduce the \emph{Geometric Observability Index} (\GOI), a per-feature
sensitivity measure for pose estimation on $\SE$: the metric norm of the pose
perturbation that a single measurement induces through the---possibly
rank-deficient---Gauss--Newton curvature, restricted to the observable
subspace. We prove that \GOI{} equals the norm of the M-estimator influence
function, that the underlying curvature operator coincides with the Fisher
information, and that its smallest observable eigenvalue $\lmin$ governs both
the worst-case amplification of a measurement's effect and a finite-sample
stability radius $O(\sigma/\sqrt{n\lmin})$. Operationally the theory cuts both
ways. \GOI{} is the exact per-measurement \emph{attribution}---to first order
it \emph{is} the leave-one-out pose shift---yet the influence standardized by
its own inlier null covariance collapses \emph{exactly} to the classical
chi-square residual statistic. Residual gating is thus the influence test
standardized against each measurement's own geometry: this explains its
robustness to residual-visible contamination and, equally, delimits its blind
spot, while raw-influence gating conflates a measurement's information with its
harm and is predicted to over-reject high-leverage inliers in weakly observable
geometry. Controlled synthetic experiments check the identities and scaling
laws numerically, and studies on five TUM RGB-D sequences (four dynamic, one
static control) and two KITTI odometry sequences test the practical
prediction: parity of the two criteria under well-conditioned geometry, and
significant degradation of raw-influence gating at
$\mathrm{cond}(\Cv)\approx10^{4}$. All code is released for reproducibility.
\end{abstract}


\section{Introduction}
\label{sec:intro}

Pose estimation on $\SE$ underlies visual odometry, SLAM, and structure-from-motion.
Every mainstream back end decides which measurements to trust by the \emph{size of
their residuals}: chi-square gating, robust kernels, and their variants are all
functions of $\norm{r}$. The premise behind this practice---that a measurement's
danger to the estimate is proportional to its residual---fails exactly where it
matters most. In weakly observable geometry, a measurement of moderate residual
whose score aligns with a poorly constrained direction shifts the pose far more than
a large residual along a well-constrained one; the amplification is $1/\lambda_{i}$
in the curvature eigenvalue, unbounded as observability degrades. The natural
conclusion---that one should gate on the measurement's \emph{influence on the
pose} rather than its residual---turns out to be exactly half right. This paper
develops the per-measurement influence into a complete theory and shows it cuts
both ways: influence is the exact \emph{attribution} of estimate motion, yet once
standardized by its inlier null covariance it collapses to the chi-square residual
statistic, vindicating the classical test from first principles and delimiting
precisely where raw influence misleads. Along the way, the three failure modes
practitioners treat with separate heuristics---\emph{dynamic} features that bias
the estimate, \emph{low-parallax} or \emph{near-pure-rotation} geometry in which
translation becomes weakly observable, and \emph{ill-conditioning} that amplifies
noise---emerge as facets of one operator-theoretic object.

Let $\Cv=\E[J^{\top}WJ]$ be the curvature (Gauss--Newton Hessian) of the weighted
least-squares pose problem, let $G$ be a metric on $\se\cong\R^{6}$, and let
$\psi(z)=J^{\top}W r(z)$ be the score contributed by a measurement $z$. Writing
$\gsc(z)=G^{-1}\psi(z)$ for the gradient representative of the score and
$\Aop=G^{-1}\Cv$ for the associated curvature operator (\cref{sec:observable}), we
define the Geometric Observability Index of $z$ as
\begin{equation}
\label{eq:goi-def}
\GOI(z)\;=\;\Gnorm{\Hoo^{-1}\PO\,\gsc(z)},
\end{equation}
the $G$-magnitude of the pose perturbation that the single measurement induces. The
contributions of this paper are the following.

\begin{enumerate}
\item \textbf{Influence: exact attribution, and the correct standardization.}
We prove that \GOI(z) equals the $G$-norm of the M-estimator influence function
(\cref{thm:if})---the first-order pose shift a measurement causes---and check
the identity numerically against the re-solved estimator and against the true
leave-one-out pose shift (\cref{sec:exp1}). We further prove that the influence
standardized by its own inlier null covariance equals \emph{exactly} the
classical chi-square residual statistic (\cref{prop:student}). The content of
this identity is a division of roles rather than a new detector: residual gating
is the influence test standardized against each measurement's own geometry,
which explains its robustness to residual-visible contamination and delimits
its blind spot (\cref{rem:leverage}), while raw-influence gating conflates
information with harm and is predicted to over-reject inliers aligned with
weakly observable directions. Real data test both halves: parity of the
criteria on well-conditioned TUM scenes (\cref{sec:exp5}) and significant
degradation of raw-influence gating on KITTI at
$\mathrm{cond}(\Cv)\approx10^{4}$ (\cref{sec:exp6}). The classical ancestors
are Cook's distance and standardized residuals~\cite{cook77,hampel}; to our
knowledge neither the $\SE$ influence with rank-deficient curvature nor its
exact collapse to the chi-square test under standardization has been formulated
before.

\item \textbf{One operator, four heuristics unified.} Influence
(\cref{thm:if}), Fisher information (\cref{thm:fisher}), conditioning, and
finite-sample stability (\cref{thm:stability}) are shown to be four facets of the
spectral geometry of the single curvature operator $\Aop=G^{-1}\Cv$, grounding the
separate heuristics---robust kernels, chi-square tests, condition-number
thresholds---in one calibrated object; in particular, the chi-square test is
recovered, not replaced: it emerges as the exactly standardized influence
(\cref{prop:student}).

\item \textbf{Exact spectral form and weak-observability law.} In a $G$-orthonormal
eigenbasis of $\Cv v=\lambda Gv$ we obtain the closed form
$\GOI(z)^{2}=\sum_{i:\,\lambda_{i}>0}\psi_{i}(z)^{2}/\lambda_{i}^{2}$ with duality
coefficients $\psi_{i}(z)=v_{i}^{\top}\psi(z)$ (\cref{prop:spectral}), making the
$1/\lmin$ amplification explicit; translational collapse obeys
$\lmin\propto\text{parallax}^{2}$ (\cref{rem:parallax}), yielding the degeneracy
test (D2) and the weak-direction alignment ratio $\rho_{1}$ (D3).

\item \textbf{Finite-sample stability.} We prove a stability radius
$\Gnorm{\hat\xi_{n,O}}=O(\sigma/\sqrt{n\lmin})$, valid once
$n\gtrsim\sigma^{2}\lmin^{-3}\log(1/\delta)$ up to constants
(\cref{thm:stability}), tying the deterministic and statistical pictures together.

\item \textbf{Rigor for every metric, and validation.} The framework is stated for an
arbitrary metric $G$ via the index-raised curvature operator and its
$G$-pseudoinverse (\cref{lem:projector}), exact identities that fail for the naive
matrix formulation (\cref{rem:index}); every identity and scaling law is checked
numerically, and every practical prediction is tested on real TUM RGB-D and KITTI
sequences, with a released, self-contained implementation
(\cref{sec:experiments}).
\end{enumerate}

\paragraph{Organization.}
\Cref{sec:related} reviews related work. \Cref{sec:prelim} fixes the $\SE$ geometry,
the score, and the curvature operator. \Cref{sec:observable} defines the observable
subspace $O$, the restricted curvature $\Hoo$, and the $G$-orthonormal spectral
calculus. \Cref{sec:goi} defines the Geometric Observability Index and its spectral
form. \Cref{sec:characterizations} proves the influence-function and Fisher
characterizations, the standardization result (\cref{prop:student}: standardized
influence equals the chi-square statistic), and the conditioning analysis.
\Cref{thm:stability} in \cref{sec:stability} gives the stability radius.
\Cref{sec:diagnostics} presents the diagnostics; \cref{sec:experiments} contains
the synthetic validation and the real-data studies on TUM RGB-D
(\cref{sec:exp5}) and KITTI (\cref{sec:exp6}); \cref{sec:discussion} discusses
scope and limitations.

\section{Related work}
\label{sec:related}

\paragraph{Influence functions and robust statistics.}
The influence function was introduced by Hampel~\cite{hampel74} and developed into a
systematic theory of robustness in~\cite{hampel,huber}; Cook's
distance~\cite{cook77} is its leave-one-out analogue in regression, and influence
functions have recently been revived as a diagnostic for machine-learning
models~\cite{kohliang}. These treatments are Euclidean and assume an invertible
information matrix. Our contribution is the extension to left-trivialized
M-estimators on $\SE$ with a possibly rank-deficient curvature, where the influence
must be defined on the observable subspace and its magnitude interacts with the
spectral geometry of the curvature operator.

\paragraph{Robust back ends and dynamic features in SLAM.}
Modern estimators gate measurements by residual magnitude: robust kernels in bundle
adjustment~\cite{triggs}, chi-square residual tests in
ORB-SLAM2/3~\cite{orbslam2,orbslam3}, and outlier-aware factor-graph solvers such as
g2o~\cite{g2o}, iSAM2~\cite{isam2}, and GTSAM~\cite{dellaert}. Dynamic-scene systems
such as DynaSLAM~\cite{dynaslam} and DS-SLAM~\cite{dsslam} add semantic or
geometric motion cues. All of these criteria are functions of the residual; observability-aware feature
\emph{selection}, which ranks candidate measurements by their contribution to the
information spectrum~\cite{carlone2019attention,zhao2020good}, is complementary
(it measures informativeness, not harm). The
framework here instead measures the measurement's \emph{influence on the pose};
\cref{prop:student} shows the chi-square test is exactly this influence
standardized against the measurement's own inlier geometry; the raw influence,
by contrast, can differ from residual magnitude by orders of magnitude in
ill-conditioned geometry (\cref{sec:exp1}).

\paragraph{Observability and consistency analysis.}
Observability-based consistency rules for EKF-SLAM~\cite{huang2010} and
vision-aided inertial navigation~\cite{hesch2014}, and nonlinear observability
analyses in the sense of~\cite{martinelli}, characterize \emph{which} directions are
unobservable at the system level. We quantify weak observability spectrally, through
$\lmin$ of the curvature operator, and connect it in a single framework to
per-feature influence (\cref{thm:if}), Fisher information (\cref{thm:fisher}), and
finite-sample stability (\cref{thm:stability}).

\paragraph{Estimation and optimization on Lie groups.}
Our geometric setting follows the state-estimation formalism
of~\cite{barfoot,chirikjian,sola}, invariant filtering~\cite{barrau}, and
manifold optimization~\cite{absil,boumal}. The technical novelty here is the
operator-theoretic treatment of the metric: distinguishing the curvature bilinear
form from the $G$-self-adjoint curvature operator obtained by raising an index
(\cref{sec:observable}) is what makes the projector and pseudoinverse identities of
\cref{lem:projector} exact for every metric $G$.

\section{Preliminaries}
\label{sec:prelim}

\paragraph{Group and algebra.}
A pose is $g=(R,t)\in\SE$ with $R\in\mathrm{SO}(3)$, $t\in\R^{3}$, acting on a world
point $X\in\R^{3}$ by $g^{-1}X=R^{\top}(X-t)$. We use \emph{left-trivialized} (spatial-frame)
perturbations: for $\xi=(\rho,\phi)\in\se\cong\R^{6}$, with $\rho$ the
translational and $\phi$ the rotational part,
\begin{equation}
g(\xi)=\exp(\xi)\,g,\qquad
\xi^{\wedge}=\begin{bmatrix}\phi^{\wedge}&\rho\\[1pt]0&0\end{bmatrix},
\end{equation}
where $(\cdot)^{\wedge}:\R^{3}\to\so$ is the usual hat map. We endow $\se$ with a fixed
inner product $\Ginner{\cdot}{\cdot}=\cdot^{\top}G\,\cdot$ for a symmetric
positive-definite $G\in\R^{6\times6}$; $\Gnorm{\cdot}$ denotes the induced norm,
$\norm{\cdot}$ the Euclidean norm, and $\opnorm{\cdot}$ the operator (spectral) norm
throughout.

\paragraph{Measurement model.}
A calibrated perspective measurement of $X$ is $z=\pi(g^{-1}X)+\eta$, where
$\pi(p)=(p_{1}/p_{3},\,p_{2}/p_{3})$ and $\eta\sim(0,\Sigma)$. The residual and its
left-trivialized Jacobian are
\begin{equation}
\begin{gathered}
r(z,g)=z-\pi(g^{-1}X),\\
J=\left.\frac{\partial r(z,\exp(\xi)g)}{\partial\xi}\right|_{\xi=0}
=\frac{\partial\pi}{\partial p}\,[\,R^{\top}\;\;-R^{\top}X^{\wedge}\,]\in\R^{2\times6}.
\end{gathered}
\end{equation}

\paragraph{Score and curvature.}
The per-measurement score and the population curvature (Gauss--Newton) matrix are
\begin{equation}
\label{eq:score-curv}
\psi(z)=J^{\top}W\,r(z)\in\R^{6},\qquad
\Cv=\E\!\left[J^{\top}W J\right]\in\R^{6\times6},
\end{equation}
the expectation taken over the feature distribution. Its empirical counterpart over
$n$ features is $\Cv_{n}=\tfrac1n\sum_{i=1}^{n}J_{i}^{\top}W J_{i}$. The score is
naturally a \emph{covector} (it pairs with pose perturbations through the duality
pairing $\psi^{\top}\xi$); we write
\begin{equation}
\label{eq:grad-score}
\gsc(z):=G^{-1}\psi(z)
\end{equation}
for its gradient (Riesz) representative, so that $\Ginner{\gsc(z)}{u}=\psi(z)^{\top}u$
for all $u\in\se$. For $G=I$ the two coincide. We collect the standing regularity
assumptions used in the proofs.

\begin{assumption}[Regularity]
\label{as:reg}
\emph{(A1)} $g\mapsto J(g)$ is $C^{1}$ on a neighborhood of the true pose $g^{*}$,
and its derivative $D_{g}J$ (equivalently, the second derivative of the residual) is
bounded in operator norm by $C_{J'}$.
\emph{(A2)} $\opnorm{J}\le C_{J}$ and $\opnorm{W}\le C_{W}$ uniformly over features.
\emph{(A3)} The noise is zero-mean with covariance $\Sigma=\sigma^{2}\Sigma_{0}$ for
a fixed normalized $\Sigma_{0}$, is sub-Gaussian (in particular has finite moments
of all orders; Gaussian noise, as in the experiments, suffices), and is independent
of the feature geometry. Unless stated otherwise we take $W=\Sigma^{-1}$ (so one may
normalize $\sigma=1$); \cref{thm:stability} makes the $\sigma$-dependence explicit by
working with the noise-free weight $W=\Sigma_{0}^{-1}$.
\emph{(A4)} (Used in \cref{thm:stability} only.) Features are drawn i.i.d.\ from the
distribution defining the expectation in \eqref{eq:score-curv}, independently of the
noise, and there is a \emph{leverage constant} $\kappa\ge1$ such that
$v^{\top}\!\big(J^{\top}WJ\big)v\le\kappa\,v^{\top}\Cv\,v$ for almost every feature
and all $v\in\se$.
\end{assumption}

Condition (A4) states that no single feature carries more than $\kappa$ times the
population-average information in any direction. It always holds with
$\kappa\le C_{J}^{2}C_{W}/\lmin$ by (A2), but for well-spread feature sets---every
direction informed by a constant fraction of the features---$\kappa=O(1)$
independently of the conditioning, and it is in this regime that
\cref{thm:stability} is informative; \cref{sec:exp4,sec:exp5,sec:exp6} measure
$\kappa$ on every design and dataset the theorem is tested on. Unlike $\lmin$,
the leverage ratio is invariant under a change of chart (both sides transform
by the same congruence), so its measurement is insensitive to the gauge choice
of \cref{rem:gauge}. On $\ker\Cv$ both sides of (A4) vanish
(if $v^{\top}\Cv v=0$ then $Jv=0$ almost surely), so the condition constrains only
the observable directions.

\section{The observable subspace and restricted curvature}
\label{sec:observable}

The curvature matrix $\Cv$ is symmetric positive semidefinite but, under weak
observability, need not be invertible. Two distinct objects must be kept apart: the
curvature \emph{bilinear form}, represented by the matrix $\Cv$, and the associated
\emph{operator} on the inner-product space $(\se,\Ginner{\cdot}{\cdot})$,
\begin{equation}
\label{eq:curv-op}
\Aop:=G^{-1}\Cv,
\end{equation}
obtained by raising an index with the metric. The operator $\Aop$ is self-adjoint
and positive semidefinite with respect to $\Ginner{\cdot}{\cdot}$:
$\Ginner{\Aop u}{w}=u^{\top}\Cv w=\Ginner{u}{\Aop w}$ and
$\Ginner{\Aop u}{u}=u^{\top}\Cv u\ge0$. All spectral statements below are statements
about $\Aop$; for $G=I$ one has $\Aop=\Cv$ and the distinction disappears.

\begin{definition}[Observable subspace]
\label{def:O}
The \emph{observable subspace} is $O:=\range(\Aop)\subseteq\se$. Since $\Aop$ is
$G$-self-adjoint, $O^{\perp_{G}}=\ker(\Aop)=\ker(\Cv)$. We write $\PO$ for the
$G$-orthogonal projector onto $O$ and define the \emph{restricted curvature
operator} $\Hoo:=\Aop|_{O}$, the restriction of $\Aop$ to $O$ regarded as an
invertible operator $O\to O$.
\end{definition}

Because $O=\range(\Aop)$ is $\Aop$-invariant and $G$-orthogonal to $\ker\Aop$, the
restriction and the compression coincide and the inverse extends consistently:
\begin{lemma}[Projector identity]
\label{lem:projector}
Let $\pinvG$ denote the pseudoinverse of $\Aop$ in the $G$-inner product (the unique
operator satisfying the four Moore--Penrose relations with adjoints taken with
respect to $\Ginner{\cdot}{\cdot}$). Then
\[
\begin{gathered}
\PO\,\pinvG\,\PO=\Hoo^{-1}\ \text{on }O,\qquad
\PO\,\Aop\,\PO=\Hoo,\\
\Hoo^{-1}\PO=\pinvG .
\end{gathered}
\]
In the eigenbasis of \eqref{eq:geneig} below,
$\pinvG=\sum_{i:\lambda_{i}>0}\lambda_{i}^{-1}v_{i}v_{i}^{\top}G$ and
$\PO=\sum_{i:\lambda_{i}>0}v_{i}v_{i}^{\top}G$; equivalently, in whitened
coordinates,
\begin{equation}
\label{eq:whitening}
\pinvG=G^{-1/2}\,\widetilde{\Cv}^{\dagger}\,G^{1/2},\qquad
\widetilde{\Cv}:=G^{-1/2}\Cv\,G^{-1/2},
\end{equation}
with $\widetilde{\Cv}^{\dagger}$ the ordinary Moore--Penrose pseudoinverse. For
$G=I$, $\Aop=\Cv$ and $\pinvG=\Cv^{\dagger}$.
\end{lemma}

\begin{proof}
$\Aop$ is self-adjoint and positive semidefinite on the finite-dimensional
inner-product space $(\se,\Ginner{\cdot}{\cdot})$, so the spectral theorem provides
a $G$-orthonormal eigenbasis in which $\Aop$ is diagonal with nonnegative entries;
$O=\range(\Aop)$ is spanned by the eigenvectors with positive eigenvalues and is
$G$-orthogonal to $\ker\Aop$. In this basis $\pinvG$ inverts the positive block and
annihilates the kernel, which is exactly $\Hoo^{-1}$ pre- and post-composed with
$\PO$; the displayed matrix expressions are the same statement written in
coordinates. For \eqref{eq:whitening}, the map $u\mapsto G^{1/2}u$ is an isometry
from $(\se,\Ginner{\cdot}{\cdot})$ to Euclidean $\R^{6}$ that conjugates $\Aop$ into
the symmetric matrix $\widetilde{\Cv}=G^{1/2}\Aop G^{-1/2}$; pseudoinversion
commutes with isometric conjugation, giving \eqref{eq:whitening}.
\end{proof}

\begin{remark}[Why the index must be raised]
\label{rem:index}
For a general metric, the Euclidean objects attached to the \emph{matrix} $\Cv$ are
the wrong ones: $(\range\Cv)^{\perp_{G}}=G^{-1}\ker\Cv\neq\ker\Cv$, and
$\PO\Cv^{\dagger}\PO\neq\Hoo^{-1}$, unless $G\Cv=\Cv G$ (e.g.\ $G=\alpha I$).
Identifying the curvature form with an operator without composing with $G^{-1}$
silently assumes this commutation. Fixing $\Aop=G^{-1}\Cv$ and $O=\range(\Aop)$
throughout removes the ambiguity, and we use \cref{lem:projector} freely below.
\end{remark}

\paragraph{$G$-orthonormal spectral calculus.}
The eigenproblem of $\Aop$ is the generalized eigenproblem
\begin{equation}
\label{eq:geneig}
\begin{gathered}
\Aop v_{i}=\lambda_{i}v_{i}
\;\Longleftrightarrow\;
\Cv v_{i}=\lambda_{i}\,G\,v_{i},\\
\Ginner{v_{i}}{v_{j}}=v_{i}^{\top}G v_{j}=\delta_{ij},
\end{gathered}
\end{equation}
with eigenvalues ordered $\lambda_{1}\ge\dots\ge\lambda_{6}\ge0$. The eigenvectors
$\{v_{i}\}$ are $G$-orthonormal and $O=\operatorname{span}\{v_{i}:\lambda_{i}>0\}$. For
any $u\in\se$ we write its $G$-coordinates $u_{i}=\Ginner{u}{v_{i}}=v_{i}^{\top}Gu$,
so that $u=\sum_{i}u_{i}v_{i}$ and $\Gnorm{u}^{2}=\sum_{i}u_{i}^{2}$. For a
\emph{covector} $w$ (such as the score $\psi$) the coordinates of its gradient
representative $G^{-1}w$ are the duality pairings
$\Ginner{G^{-1}w}{v_{i}}=v_{i}^{\top}w$; no metric factor appears. We adopt this
$G$-orthonormal basis \emph{everywhere}.

\section{The Geometric Observability Index}
\label{sec:goi}

\begin{definition}[Geometric Observability Index]
\label{def:goi}
For a measurement $z$ with score $\psi(z)$ and gradient score
$\gsc(z)=G^{-1}\psi(z)$, the Geometric Observability Index is
\begin{equation}
\GOI(z)=\Gnorm{\Hoo^{-1}\PO\,\gsc(z)}=\Gnorm{\pinvG\,\gsc(z)},
\end{equation}
the second equality by \cref{lem:projector}. For $G=I$ this reduces to
$\GOI(z)=\norm{\Cv^{\dagger}\psi(z)}$.
\end{definition}

\begin{proposition}[Spectral form]
\label{prop:spectral}
In the $G$-orthonormal eigenbasis \eqref{eq:geneig}, writing the duality
coefficients $\psi_{i}(z)=v_{i}^{\top}\psi(z)=\Ginner{\gsc(z)}{v_{i}}$,
\begin{equation}
\label{eq:goi-spectral}
\GOI(z)^{2}=\sum_{i:\,\lambda_{i}>0}\frac{\psi_{i}(z)^{2}}{\lambda_{i}^{2}}.
\end{equation}
Consequently $\GOI(z)\ge|\psi_{\imath}(z)|/\lmin$ for the weakest observable direction
$\imath=\arg\min_{i:\lambda_i>0}\lambda_i$, with equality iff $\PO\gsc(z)$ is
$G$-aligned with $v_{\imath}$.
\end{proposition}

\begin{proof}
Expanding in the eigenbasis of \cref{lem:projector},
$\pinvG\gsc=\sum_{i:\lambda_{i}>0}\lambda_{i}^{-1}\Ginner{\gsc}{v_{i}}\,v_{i}
=\sum_{i:\lambda_{i}>0}\lambda_{i}^{-1}\psi_{i}\,v_{i}$; taking $\Gnorm{\cdot}^{2}$
and using $G$-orthonormality gives \eqref{eq:goi-spectral}. The bound keeps only the
$\imath$ term.
\end{proof}

Equation \eqref{eq:goi-spectral} is the analytical core: \GOI{} amplifies score
energy lying along weakly observable directions by $1/\lambda_{i}^{2}$, so a feature
of modest residual can dominate the pose update if it aligns with $v_{\imath}$. We
quantify this alignment by
\begin{equation}
\label{eq:rho1}
\rho_{1}(z)=\frac{|\psi_{\imath}(z)|}{\Gnorm{\PO\gsc(z)}}\in[0,1].
\end{equation}

\section{Operator-theoretic characterizations}
\label{sec:characterizations}

\subsection{Influence function}
\label{sec:if}

Consider the M-estimator $\hat g$ solving $\E_{P}[\psi(z,g)]=0$ for a measurement law
$P$, and the contaminated law $P_{\varepsilon}=(1-\varepsilon)P+\varepsilon\delta_{z}$.
Let $g_{\varepsilon}$ denote the corresponding solution and
$\eta:=\dot g_{\varepsilon}|_{\varepsilon=0}\in\se$ its left-trivialized derivative.

\begin{theorem}[Influence function]
\label{thm:if}
Under \cref{as:reg}, the left-trivialized influence function of the pose M-estimator at
$z$ is
\begin{equation}
\begin{gathered}
\IF(z)=-\,\Hoo^{-1}\PO\,\gsc(z)\in O,\\
\text{hence}\quad
\GOI(z)=\Gnorm{\IF(z)}.
\end{gathered}
\end{equation}
\end{theorem}

\begin{proof}
Solutions are sought for the $O$-projected estimating equation
$\Ginner{(1-\varepsilon)\,G^{-1}\E_{P}[\psi]+\varepsilon\,G^{-1}\psi(z,\cdot)}{v_{i}}=0$
($\lambda_{i}>0$) on the submanifold $\exp(O)g^{*}$---the identifiable component of
the problem. This loses nothing on the population side: if $v\in\ker\Cv$ then
$\E\norm{Jv}_{W}^{2}=0$, so $Jv=0$ and $v^{\top}\psi=(Jv)^{\top}Wr=0$ almost surely,
i.e.\ the population score carries no kernel component (a contaminant may; $\PO$
discards it, consistently with the definition of $\IF$ below). By (A1) the projected
equation is $C^{1}$ and its Jacobian in the $O$-directions at $(\varepsilon,\xi)=(0,0)$
is $\Hoo$, which is invertible, so the implicit function theorem yields a $C^{1}$ path
$\varepsilon\mapsto g_{\varepsilon}\in\exp(O)g^{*}$; when $\ker\Aop\neq\{0\}$ the
full solution set is a coset of $\exp(\ker\Aop)$ and $\IF$ denotes the derivative of
this $O$-component selection.
Differentiating the estimating equation
$(1-\varepsilon)\E_{P}[\psi(\cdot,g_{\varepsilon})]+\varepsilon\,\psi(z,g_{\varepsilon})=0$
at $\varepsilon=0$ and using $\E_{P}[\psi(\cdot,g^{*})]=0$ gives
$\E_{P}[D_{g}\psi]\,\eta+\psi(z,g^{*})=0$. By \eqref{eq:score-curv} and \cref{as:reg},
$\E_{P}[D_{g}\psi]=+\Cv$: indeed
$D_{g}\psi=(D_{g}J)^{\top}Wr+J^{\top}WJ$, the first term has zero mean at $g^{*}$
(since $\E[r\,|\,\text{feature}]=0$ and $D_{g}J$ is deterministic given the feature)
and the second has expectation $\Cv$---the Gauss--Newton identification of the
expected score Jacobian with the curvature, i.e.\ the Hessian of the loss
$\tfrac12 r^{\top}Wr$ up to the zero-mean residual-curvature term. Hence
$\Cv\,\eta=-\psi(z,g^{*})$, equivalently $\Aop\,\eta=-\gsc(z)$. The component of
$\eta$ in $O^{\perp_{G}}=\ker\Aop$ is unconstrained by the data and excluded by
definition; restricting to $O$ and inverting $\Hoo$ yields
$\eta=-\Hoo^{-1}\PO\gsc(z)=\IF(z)$. Taking $\Gnorm{\cdot}$ gives the \GOI{} identity.
\end{proof}

\Cref{thm:if} is the operational meaning of \GOI: it is exactly the $G$-magnitude of
the pose shift a measurement causes, to first order. This is the claim we test most
directly in \cref{sec:exp1}.

\subsection{Fisher information and the curvature operator}
\label{sec:fisher}

\begin{theorem}[Fisher--curvature identity]
\label{thm:fisher}
Under Gaussian noise with $W=\Sigma^{-1}$ and \cref{as:reg}, the Fisher information of
the pose equals the curvature,
\begin{equation}
\label{eq:fisher}
I(g^{*})=\E\!\left[\psi\psi^{\top}\right]=\E\!\left[J^{\top}W J\right]=\Cv .
\end{equation}
Here $I(g^{*})$ is the information carried by one measurement. In the
$G$-orthonormal eigenbasis, the Cram\'er--Rao bound for any unbiased estimator
of the observable coordinates $t_{i}=\Ginner{\xi}{v_{i}}$ from $n$ i.i.d.\
features reads
$\mathrm{Cov}(\hat t\,)\succeq\tfrac1n\operatorname{diag}(1/\lambda_{i})_{\lambda_{i}>0}$;
on $O^{\perp_{G}}$ the information is zero and no bound is available (the CRB must be
read as restricted to $O$).
\end{theorem}

\begin{proof}
At $g^{*}$ the residual is pure noise, $r=\eta$, with
$\E[rr^{\top}\,|\,\text{feature}]=\Sigma$ and the noise independent of the feature
geometry (A3). Conditioning on the feature (which determines $J$) and using
$W\Sigma W=W$,
\[
\begin{aligned}
\E[\psi\psi^{\top}]
&=\E\!\big[J^{\top}W\,\E[rr^{\top}\,|\,J]\,W J\big]\\
&=\E[J^{\top}W\Sigma W J]=\E[J^{\top}WJ]=\Cv .
\end{aligned}
\]
The Fisher information of the Gaussian measurement model is
$I(g^{*})=\E[J^{\top}\Sigma^{-1}J]=\Cv$. For the CRB, reparametrize by the
observable coordinates $t_{i}$: the information matrix of $t$ has entries
$v_{i}^{\top}I(g^{*})v_{j}=v_{i}^{\top}\Cv v_{j}=\lambda_{i}\delta_{ij}$ per
measurement, hence $n\operatorname{diag}(\lambda_{i})$ for $n$ i.i.d.\ features,
so $\mathrm{Cov}(\hat t\,)\succeq\tfrac1n\operatorname{diag}(1/\lambda_{i})$ on
the coordinates with $\lambda_{i}>0$; on $\ker\Aop$ no information accrues.
\end{proof}

Thus the deterministic curvature and the statistical Fisher information are the same
object, and \GOI{} is simultaneously a conditioning quantity and a (restricted)
Cram\'er--Rao sensitivity. We verify \eqref{eq:fisher} entrywise in \cref{sec:exp3}.

\subsection{Standardized influence and the chi-square statistic}
\label{sec:student}

Influence measures how much a measurement moves the estimate; for outlier
\emph{testing}, the relevant question is whether it moves the estimate more than
an \emph{inlier} at the same geometric position would. The two differ by a
leverage factor, and standardizing by the measurement's own inlier null
covariance removes it exactly.

\begin{proposition}[Standardized influence is the chi-square statistic]
\label{prop:student}
Let $\Cv_{z}:=J^{\top}WJ$ denote the per-measurement curvature, assume
$\operatorname{rank}J=2$, and adopt the weight convention of \cref{thm:stability}
($\Sigma=\sigma^{2}\Sigma_{0}$, $W=\Sigma_{0}^{-1}$; for $W=\Sigma^{-1}$ set
$\sigma=1$). Under \cref{as:reg} at $g^{*}$, conditionally on the feature, the
inlier null covariance of the score is
$\mathrm{Cov}(\psi(z))=\sigma^{2}\Cv_{z}$, and the score standardized by it
collapses to the residual statistic:
\begin{equation}
\label{eq:student}
\psi(z)^{\top}\big(\sigma^{2}\Cv_{z}\big)^{\dagger}\psi(z)
\;=\;\frac{r^{\top}Wr}{\sigma^{2}}.
\end{equation}
\end{proposition}

\begin{proof}
$\E[\psi\psi^{\top}\,|\,J]=J^{\top}W\,\E[rr^{\top}]\,WJ
=\sigma^{2}J^{\top}W\Sigma_{0}WJ=\sigma^{2}\Cv_{z}$ with $W=\Sigma_{0}^{-1}$.
For the identity, write $\tilde J=W^{1/2}J$ and $\tilde r=W^{1/2}r$; then
$\psi^{\top}\Cv_{z}^{\dagger}\psi
=\tilde r^{\top}\tilde J(\tilde J^{\top}\tilde J)^{\dagger}\tilde J^{\top}\tilde r$,
and $\tilde J(\tilde J^{\top}\tilde J)^{\dagger}\tilde J^{\top}$ is the orthogonal
projector onto $\operatorname{col}(\tilde J)$, which is all of $\R^{2}$ when
$\operatorname{rank}J=2$; hence the expression equals
$\tilde r^{\top}\tilde r=r^{\top}Wr$. (If $\operatorname{rank}J<2$ the right-hand
side is the projected residual statistic.) Finally, the same value is obtained by
standardizing the \emph{influence} itself: for any linear map $M$ injective on the
support of $\psi$, writing $\mathrm{Cov}(\psi)=LL^{\top}$ with $L$ of full column
rank and $\psi=Lu$, one has
$(M\psi)^{\top}\big(M\,\mathrm{Cov}(\psi)M^{\top}\big)^{\dagger}(M\psi)
=u^{\top}(ML)^{\top}\big((ML)(ML)^{\top}\big)^{\dagger}(ML)\,u=u^{\top}u
=\psi^{\top}\mathrm{Cov}(\psi)^{\dagger}\psi$, since
$A^{\top}(AA^{\top})^{\dagger}A$ is the identity for full-column-rank $A$.
Taking $M=-\Hoo^{-1}\PO G^{-1}$, which is injective on
$\range\Cv\supseteq\operatorname{supp}(\psi)$ (the score of a
population-supported feature carries no kernel component, cf.\ the proof of
\cref{thm:if}), shows the Mahalanobis norm of $\IF(z)$ with respect to its own
null covariance equals \eqref{eq:student}.
\end{proof}

\begin{remark}[What the standardization does and does not say]
\label{rem:standardization}
Two points delimit \cref{prop:student}. (i) The last step of the proof is an
algebraic identity valid for \emph{every} injective linear map $M$: whitening
any linear image of $\psi$ by its own null covariance returns the whitened
$\psi$. The geometric content of the proposition therefore lies in the
\emph{choice} of standardization---each measurement is compared with an inlier
of the same geometry, through $\Cv_{z}$ alone---and not in a property of
$\Hoo^{-1}$; a different reference (e.g.\ the population curvature $\Cv$) yields
a different statistic. (ii) The statistic is evaluated at the true pose $g^{*}$
and normalized by the population noise covariance. It is thus the $\SE$
analogue of the classical \emph{standardized} residual, not of the
\emph{studentized} residual, which is evaluated at the fitted solution and
carries the in-sample factor $(1-h_{ii})$ of the global
leverage~\cite{cook77}. At the fitted pose a high-leverage measurement pulls the
estimate toward itself and its residual shrinks accordingly; the chi-square
gate is blind to exactly this masking mechanism, which is the price of the
standardization and is the failure mode a bounded-influence design must address
separately.
\end{remark}

\begin{remark}[Influence conflates information with harm]
\label{rem:leverage}
\Cref{prop:student} frames the practical use of \GOI. First, the ubiquitous
chi-square residual test \emph{is} the influence test standardized against the
measurement's own inlier geometry: this accounts for its robustness to
contamination that is visible in the residual, and, by
\cref{rem:standardization}, for its blindness to contamination that is not.
Second, thresholding \emph{raw} \GOI{} over-rejects high-leverage inliers: a
measurement aligned with a weakly observable direction has influence inflated by
$1/\lambda_{\imath}$ whether its residual is bias or ordinary noise, because the
very \emph{information} it carries about that direction is what makes it
influential; removing it de-informs the scarce direction. Raw-influence gating
is therefore predicted to degrade precisely in weakly observable geometry---the
prediction tested in \cref{sec:exp6} on KITTI. \GOI{} remains the correct
quantity for \emph{attribution} (which measurement moved the estimate, and by
how much, \cref{thm:if}) and for sensitivity auditing; it is the wrong raw
statistic for outlier testing.
\end{remark}

\subsection{Conditioning and weak observability}
\label{sec:conditioning}

The smallest observable eigenvalue $\lmin:=\min_{i:\lambda_{i}>0}\lambda_{i}$ measures
weak observability. By \cref{prop:spectral} the worst-case amplification of score energy
is $1/\lmin$, attained along $v_{\imath}$. In $\SE$ pose estimation the translational
block of $\Cv$ scales with squared parallax, so $\lmin$ collapses quadratically as
parallax vanishes (\cref{sec:exp2}); the eigenvector $v_{\imath}$ is then
translation-dominated, identifying \emph{which} motion is unobservable.

\begin{remark}[Choice of trivialization and origin]
\label{rem:gauge}
The spectrum of $\Aop$, and hence \GOI, depends on the chart: in the spatial
(left) trivialization the rotational block of $J$ contains the world coordinates
$X^{\wedge}$, so rotations are measured about the \emph{world origin} and the
rotation--translation coupling scales with $\norm{X}$. This is immaterial when the
scene sits near the origin, but on long trajectories (world coordinates of
$10^{2}$\,m and more) it inflates the spectrum and distorts both the index and the
$\SE$ log error metric by orders of magnitude. All quantities should therefore be
computed in a camera-anchored frame---equivalently, after recentering the world at
the current camera position, under which the residuals are invariant. Our real-data
pipelines apply this recentering per frame.
\end{remark}

\begin{remark}[Quadratic translational collapse]
\label{rem:parallax}
The scaling has an elementary source. With $p=R^{\top}(X-t)$ and depth $Z=p_{3}$,
$\partial\pi/\partial p=Z^{-1}[\,I_{2}\;\;-\pi(p)\,]$, so the translational block
$J_{\rho}=(\partial\pi/\partial p)R^{\top}$ has entries $O(1/Z)$, while the
rotational block $J_{\phi}=-(\partial\pi/\partial p)R^{\top}X^{\wedge}$ has entries
$O(\norm{X}/Z)=O(1)$ for scenes viewed at bounded field of view. The translational
block of $\Cv$ therefore scales as $Z^{-2}$: with the parallax proxy
$\propto 1/Z$ used in \cref{sec:exp2}, $\lmin\propto\text{parallax}^{2}$ \emph{asymptotically in the deep-scene regime},
while the rotational eigenvalues persist; at near range the translation--rotation
coupling makes $\lmin$ decay more slowly. \Cref{sec:exp2} fits the exponent $1.98$
over the asymptotic regime.
\end{remark}

\section{Stability radius}
\label{sec:stability}

We now bound the finite-sample pose error in terms of $\lmin$. Let
$\hat\xi_{n,O}=\PO\,\log(\hat g_{n}\,g^{*-1})$ be the observable component of the
estimation error from $n$ noisy features.

\begin{theorem}[Stability radius]
\label{thm:stability}
Under \cref{as:reg}, including the sampling and leverage condition (A4), write
$\Sigma=\sigma^{2}\Sigma_{0}$ and
compute $\Cv$, the eigenpairs $(\lambda_{i},v_{i})$, and $\lmin$ with the noise-free
weight $W=\Sigma_{0}^{-1}$, so that geometry and noise scale are decoupled; by the
computation in \cref{thm:fisher} (applied with weight $\Sigma^{-1}=\sigma^{-2}W$),
the score covariance is then $\E[\psi\psi^{\top}]=\sigma^{2}\Cv$. There are constants
$c_{1},c_{2},c_{3}$ depending only on $C_{J},C_{W},C_{J'}$, the leverage constant
$\kappa$, and the sub-Gaussian constant $K$ of the normalized noise
$\sigma^{-1}\Sigma_{0}^{-1/2}\eta$, such that the following holds with
probability at least $1-2\delta$. Write
$a_{n}:=c_{1}\sigma\sqrt{\log(12/\delta)/(n\lmin)}$, $b:=c_{2}/\lmin$, and
\begin{equation}
\label{eq:stability-n0}
n_{0}\;:=\;c_{3}\,(1+\sigma^{2})\,\lmin^{-2}\,\log(12/\delta),
\end{equation}
the sample size beyond which the empirical curvature is a controlled
perturbation of $\Cv$ on $O$ (see the proof). For $n\ge n_{0}$, every root
$\hat g_{n}$ of the empirical estimating equation in the ball
$\Gnorm{\xi}<1/(2b)$ around $g^{*}$ satisfies
\begin{equation}
\label{eq:stability}
\Gnorm{\hat\xi_{n,O}}\;\le\;a_{n}+b\,\Gnorm{\hat\xi_{n,O}}^{2},
\end{equation}
and if moreover $4a_{n}b<1$, i.e.\
\begin{equation}
\label{eq:stability-n}
n\;>\;\max\!\big\{16\,c_{1}^{2}c_{2}^{2}\,\sigma^{2}\,\lmin^{-3}\,\log(12/\delta),
\;n_{0}\big\},
\end{equation}
such a root exists, is unique in the ball, and obeys
\begin{equation}
\label{eq:stability-rate}
\Gnorm{\hat\xi_{n,O}}\;\le\;2a_{n}
\;=\;\frac{2c_{1}\,\sigma}{\sqrt{n\,\lmin}}\sqrt{\log(12/\delta)},
\end{equation}
i.e.\ $\Gnorm{\hat\xi_{n,O}}=O\!\big(\sigma/\sqrt{n\,\lmin}\big)$.
\end{theorem}

\begin{proof}
\emph{Linear term.}
The Gauss--Newton normal equations give
$\hat\xi_{n,O}=-\Hoo^{-1}\PO\,\bar\psi_{n}^{\,\sharp}+\mathrm{rem}$, where
$\bar\psi_{n}^{\,\sharp}=G^{-1}\bar\psi_{n}$ and
$\bar\psi_{n}=\tfrac1n\sum_{k=1}^{n}\psi(z_{k},g^{*})$. In the $G$-orthonormal
eigenbasis the observable coordinates are
$(\hat\xi_{n,O})_{i}=-\bar\psi_{n,i}/\lambda_{i}+\mathrm{rem}_{i}$ with
$\bar\psi_{n,i}=v_{i}^{\top}\bar\psi_{n}$, and
$\operatorname{Var}(\bar\psi_{n,i})=\sigma^{2}\lambda_{i}/n$ by the score covariance
above. Each summand $v_{i}^{\top}\psi(z_{k})=(W^{1/2}J_{k}v_{i})^{\top}W^{-1/2}
\Sigma_{0}^{-1}\eta_{k}$ is, conditionally on the feature, a fixed linear functional
of the noise, hence zero-mean sub-Gaussian with parameter
$K\sigma\,\norm{W^{1/2}J_{k}v_{i}}$, where $K$ is the sub-Gaussian constant of the
normalized noise (A3); and by the leverage condition (A4),
$\norm{W^{1/2}J_{k}v_{i}}^{2}=v_{i}^{\top}(J_{k}^{\top}WJ_{k})v_{i}
\le\kappa\,v_{i}^{\top}\Cv v_{i}=\kappa\lambda_{i}$ almost surely. Taking the
conditional moment-generating function and then the expectation over the i.i.d.\
features (A4) shows $\bar\psi_{n,i}$ is sub-Gaussian with parameter
$K\sigma\sqrt{\kappa\lambda_{i}/n}$. The two-sided sub-Gaussian tail bound and a
union bound over the at most six coordinates give, with probability at least
$1-\delta$, simultaneously for all $i$,
$|\bar\psi_{n,i}|\le c_{1}'\sigma\sqrt{\lambda_{i}\log(12/\delta)/n}$ with
$c_{1}'=K\sqrt{2\kappa}$; the factor $12=2\times6$ collects the two tails of the six
coordinates. Summing the squared coordinates,
\[
\begin{aligned}
\sum_{i:\lambda_{i}>0}\frac{\bar\psi_{n,i}^{2}}{\lambda_{i}^{2}}
&\le(c_{1}')^{2}\,\frac{\sigma^{2}\log(12/\delta)}{n}
\sum_{i:\lambda_{i}>0}\frac{1}{\lambda_{i}}\\
&\le\frac{6\,(c_{1}')^{2}\,\sigma^{2}\log(12/\delta)}{n\,\lmin},
\end{aligned}
\]
which is the linear term of \eqref{eq:stability} with $c_{1}=\sqrt6\,c_{1}'$.
(Bounding
$\Gnorm{\PO\bar\psi_{n}^{\,\sharp}}$ as a whole and dividing by $\lmin$ would be
lossy: the norm of the averaged score concentrates at the scale
$\sigma\sqrt{\tr\Lambda/n}$, governed by the \emph{largest} eigenvalues; the
coordinate-wise argument is what produces the correct $1/\lmin$ dependence.)

\emph{Remainder.}
The remainder collects two effects, both amplified by $\opnorm{\Hoo^{-1}}=1/\lmin$ on
$O$. First, the second-order Taylor term of the estimating equation: by (A1)--(A2)
the derivative of the empirical score map is $C_{J'}$-Lipschitz in operator norm on
a neighborhood of $g^{*}$, so the Taylor remainder of the score at $\hat\xi_{n,O}$
is at most $\tfrac12 C_{J'}C_{W}(1+C_{J})\Gnorm{\hat\xi_{n,O}}^{2}$ up to the
equivalence constants between $\norm{\cdot}$ and $\Gnorm{\cdot}$; applying
$\Hoo^{-1}\PO$ yields the quadratic term of \eqref{eq:stability} with a constant
$c_{2}$ of the stated dependence; $c_{2}$ is taken large enough that the ball
$\Gnorm{\xi}<1/(2b)$ lies inside the neighborhood of (A1). Second, the
perturbation of the derivative of the empirical estimating map at $g^{*}$.
Writing
$F_{n}(\xi):=\Hoo^{-1}\PO G^{-1}\tfrac1n\sum_{k}\psi(z_{k},\exp(\xi)g^{*})$ on
$O$, one has
\[
F_{n}'(0)-I=\Hoo^{-1}\PO G^{-1}\Big[(\Cv_{n}-\Cv)+\tfrac1n\sum_{k}(D_{g}J_{k})^{\top}W r_{k}\Big].
\]
The summands $J_{k}^{\top}WJ_{k}$ of $\Cv_{n}$ are i.i.d.\ and bounded by
$C_{J}^{2}C_{W}$ in operator norm (A2), so by the matrix Bernstein
inequality~\cite{tropp} $\opnorm{\Cv_{n}-\Cv}\le c\,C_{J}^{2}C_{W}
\sqrt{\log(12/\delta)/n}$ with probability at least $1-\delta/2$, for an
absolute constant $c$. The second bracket is a zero-mean average: conditionally
on the features each term is a linear image of the noise $r_{k}=\eta_{k}$ under
a map of norm at most $C_{J'}C_{W}$, hence sub-Gaussian with parameter
$K\sigma C_{J'}C_{W}$, and a vector sub-Gaussian tail bound gives
$\norm{\tfrac1n\sum_{k}(D_{g}J_{k})^{\top}Wr_{k}}\le c\,K\sigma C_{J'}C_{W}
\sqrt{\log(12/\delta)/n}$ with probability at least $1-\delta/2$. Both terms are
amplified by $\opnorm{\Hoo^{-1}}=1/\lmin$, so on the intersection of the two
events---probability at least $1-\delta$---$\opnorm{F_{n}'(0)-I}\le\tfrac14$
as soon as $n\ge n_{0}$ with $c_{3}\propto C_{J}^{4}C_{W}^{2}+K^{2}C_{J'}^{2}C_{W}^{2}$
(the factor $1+\sigma^{2}$ in \eqref{eq:stability-n0} collects the two terms).
On this event $F_{n}'(0)$ is invertible on $O$ with
$\beta:=\opnorm{F_{n}'(0)^{-1}}\le\tfrac43$; multiplying the estimating equation
by $F_{n}'(0)^{-1}$ inflates the linear term and the Taylor constant by at most
$\beta$, and this factor is absorbed into $c_{1}$ and $c_{2}$. Together with the
score event of the linear term, which has probability at least $1-\delta$, this
establishes \eqref{eq:stability} with probability at least $1-2\delta$ for any
root in the ball $\Gnorm{\xi}<1/(2b)$.

\emph{Localization.}
Inequality \eqref{eq:stability} is implicit in $t:=\Gnorm{\hat\xi_{n,O}}$ and by
itself does not bound $t$: the set $\{t\ge0:\,t\le a_{n}+bt^{2}\}$ is
$[0,t_{-}]\cup[t_{+},\infty)$ with
$t_{\pm}=\big(1\pm\sqrt{1-4a_{n}b}\,\big)/(2b)$ when $4a_{n}b<1$. Since
$t_{-}\le2a_{n}$ and $t_{+}\ge1/(2b)$, a root that lies in the ball
$\Gnorm{\xi}<1/(2b)$ must satisfy $t\le t_{-}\le2a_{n}$, which is
\eqref{eq:stability-rate}. Existence and uniqueness of such a root under
\eqref{eq:stability-n} follow from the Newton--Kantorovich
theorem~\cite{ortega} applied, on the event of the previous paragraph, to the
preconditioned map $\tilde F_{n}:=F_{n}'(0)^{-1}F_{n}$ at $\xi=0$: its
derivative at $0$ is the identity on $O$, its derivative is $2b$-Lipschitz on the
ball (the Taylor bound, times $\beta$, absorbed into $c_{2}$), and the Newton
step at $\xi=0$ has norm $\Gnorm{\tilde F_{n}(0)}\le a_{n}$ (the linear term,
times $\beta$, absorbed into $c_{1}$); the Kantorovich condition
$2b\cdot a_{n}\le\tfrac12$ is exactly $4a_{n}b\le1$, and it yields a root within
distance $t_{-}\le2a_{n}$ of $g^{*}$ that is unique in the ball of radius
$t_{+}\ge1/(2b)$, to which the Gauss--Newton iteration converges from any point
of that ball. The roots of $\tilde F_{n}$ and of the empirical estimating
equation coincide. Condition
\eqref{eq:stability-n} is the quantitative form of ``the quadratic term is at
most half the linear term'', i.e.\ $\Gnorm{\hat\xi_{n,O}}\lesssim\lmin$.
\end{proof}

Two readings of \eqref{eq:stability} coexist. The deterministic worst case
$\Gnorm{\hat\xi}\le\opnorm{\Hoo^{-1}}\,\Gnorm{\PO\bar\psi_{n}^{\,\sharp}}
=\Gnorm{\PO\bar\psi_{n}^{\,\sharp}}/\lmin$
scales as $1/\lmin$. The statistical RMS error, with isotropic noise spread across
directions, scales as $1/\sqrt{\lmin}$ because $\E\Gnorm{\hat\xi_{n,O}}^{2}
=\tfrac{\sigma^{2}}{n}\tr(\Hoo^{-1})\sim\sigma^{2}/(n\lmin)$. The experiments in
\cref{sec:exp4} probe the statistical law $\sigma/\sqrt{n\lmin}$.

\section{Diagnostics}
\label{sec:diagnostics}

\Cref{thm:if,thm:fisher,thm:stability,prop:student} yield three drop-in
diagnostics for a Gauss--Newton back end, each a thresholded spectral quantity
computable from $\Cv_{n}$ and the per-feature scores at negligible cost.

\paragraph{(D1) Influence attribution, and gating.}
A feature whose true 3D point has moved produces a biased residual
$r_{d}=r_{s}+b_{d}$. By \cref{thm:if} its effect on the \emph{pose} is $\GOI(z)$,
not the raw score magnitude $\norm{\psi(z)}$: a large residual along a
well-observed direction is harmless, whereas a moderate residual along
$v_{\imath}$ is not. For \emph{attribution}---identifying which measurements
moved the estimate, auditing a pose jump, or ranking sensitivity---threshold
$\GOI(z)>\tau_{\GOI}$. For outlier \emph{gating}, \cref{prop:student} prescribes
the standardized statistic, which is exactly the classical chi-square residual
test (with the masking caveat of \cref{rem:standardization}); raw-\GOI{} gating
over-rejects high-leverage inliers under weak observability
(\cref{rem:leverage,sec:exp6}).

\paragraph{(D2) Low-parallax / pure-rotation segments.}
Translational observability collapse is detected by
$\lmin(\Hoo)<\tau_{\lambda}$, with the offending direction read from
$v_{\imath}$ (translation- vs.\ rotation-dominated).

\paragraph{(D3) Weak-direction risk.}
The alignment ratio $\rho_{1}(z)$ in \eqref{eq:rho1} flags measurements whose score
concentrates on the weakest observable direction, independently of magnitude.

\section{Numerical experiments}
\label{sec:experiments}

We validate the theory on controlled synthetic $\SE$ problems and probe the
gating question on real TUM RGB-D and KITTI sequences
(\cref{sec:exp5,sec:exp6}). World points are sampled
in front of a known pose; calibrated measurements are corrupted by Gaussian pixel noise
of scale $\sigma$; the curvature, scores, and spectra are computed exactly from
\cref{sec:prelim,sec:observable}. The analytic Jacobian matches central finite
differences to $8.6\times10^{-11}$, and the included self-test reproduces the
Fisher--curvature identity \eqref{eq:fisher} to $3.3\times10^{-3}$ relative error (the
larger Monte-Carlo study of \cref{sec:exp3} attains $1.8\times10^{-3}$), so the
implementation is a faithful realization of the model. The finite-difference test
validates $J$ against the projection model alone, independently of the influence
machinery; the experiments then test the operator identities \emph{given} $J$ (this
division matters: a Gauss--Newton solver and the influence formula share the same
estimating equation, so agreement between them cannot by itself certify $J$).
We distinguish two kinds of experiment. The checks of the influence-function
and Fisher--curvature identities (\cref{sec:exp1}(a),(b) and \cref{sec:exp3}(a))
are numerical verifications of first-order identities; they certify the
implementation and the adequacy of the first-order approximation at the
contamination levels used, and a correct implementation cannot fail them. The
scaling laws (\cref{sec:exp2,sec:exp4}), whose exponents are not identities but
consequences of the model, and the real-data studies (\cref{sec:exp5,sec:exp6}),
which test the practical prediction of \cref{prop:student}, are the
substantive tests.
Experiments 1, 2, and 4 use the $\sigma$-free weight $W=\Sigma_{0}^{-1}=I_{2}$ of
\cref{thm:stability}; experiment~3(a) and the self-test use $W=\Sigma^{-1}$ as in
\cref{thm:fisher}. ``Log-correlation'' denotes the Pearson correlation between
logarithms; correlations and slopes reported as $1.00$ are rounded to two
decimals (they verify first-order identities, not statistical coincidences). All quantities are metric-$G$
with $G=I_{6}$; in this normalization $\gsc=\psi$, $\Aop=\Cv$, and $\pinvG$ is the
ordinary Moore--Penrose pseudoinverse, so the formulas below are stated with $\psi$.
Code is released (\cref{sec:repro}).

\subsection{The influence-function law and pose influence}
\label{sec:exp1}

\Cref{fig:exp1}(a) tests \cref{thm:if} directly: for an $\varepsilon$-contaminating
measurement we compare the predicted $\GOI(z)=\Gnorm{\Hoo^{-1}\PO\psi}$ against the
empirical pose shift $\Gnorm{\Delta g}/\varepsilon$ of the re-solved M-estimator. The
two agree with Pearson correlation $r=1.00$ and unit slope ($1.00$), confirming that
the implementation realizes \cref{thm:if} and that the first-order approximation
is accurate at the contamination levels used (given the independently validated
$J$). \Cref{fig:exp1}(b) illustrates the
attribution claim. Over the contaminated features of an ill-conditioned scene,
\GOI{} tracks the true leave-one-out pose shift with log-correlation $r=1.00$ in
every design we tested. This agreement is expected by construction---by
\cref{thm:if} the leave-one-out shift is, to first order, $\Hoo^{-1}\PO\gsc$,
so any criterion lacking the $\Hoo^{-1}$ factor must lose correlation in
ill-conditioned geometry---and the panel is included to exhibit the size of the
effect, not as evidence for a gating criterion. The comparison statistic is the
raw score norm $\norm{\psi}=\norm{J^{\top}Wr}$, which shares the residual
criterion's lack of the $\Hoo^{-1}$ factor. It is design-dependent, because it
conflates bias \emph{magnitude} with weak-direction \emph{alignment}: when bias
magnitudes are confined to a half-decade band---isolating the alignment
mechanism---$\norm{\psi}$ achieves only $r=0.39$; when magnitudes vary freely
over $1.4$ decades it recovers correlation through the common magnitude factor
($r=0.81$) yet cannot resolve \emph{which} equally sized biases are harmful.
Influence is not residual magnitude---as an \emph{attribution}. Whether it is the
better \emph{test} is a different question, settled by \cref{prop:student} and the
real-data studies of \cref{sec:exp5,sec:exp6}. \Cref{fig:exp1}(c) verifies the spectral mechanism of
\cref{prop:spectral}: in a low-parallax scene ($\mathrm{cond}(\Hoo)\approx2.5\times10^{4}$),
$\GOI(z)$ is bounded below by and closely tracks the weakest-direction term
$|\psi_{\imath}|/\lmin$.

\begin{figure}[t]
\centering
\includegraphics[width=\textwidth]{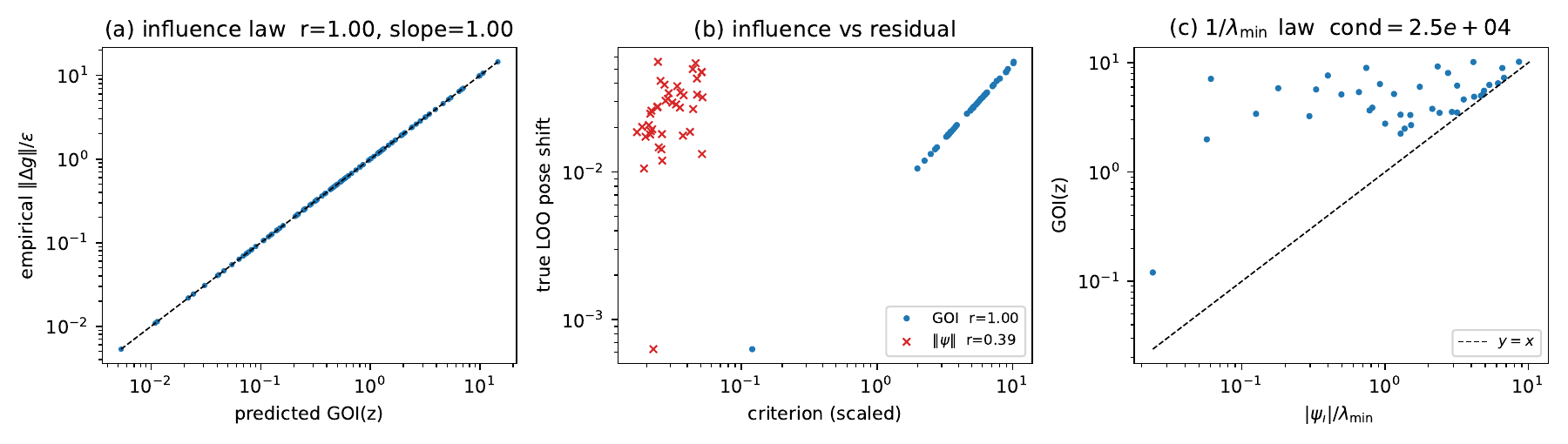}
\caption{Influence-function validation. (a) Empirical pose shift vs.\ predicted \GOI{}
($r=1.00$, slope $1.00$). (b) \GOI{} predicts the true leave-one-out pose shift
($r=1.00$) while the raw score norm does not ($r=0.39$; magnitude-controlled
contamination, see text). (c) The $1/\lmin$
amplification law: \GOI{} is governed by the weakest-direction term.}
\label{fig:exp1}
\end{figure}

\subsection{Translational collapse and pure rotation}
\label{sec:exp2}

\Cref{fig:exp2}(a) sweeps parallax (inverse mean scene depth) and measures
$\lmin(\Hoo)$. A log--log fit gives a collapse exponent of $1.98$, matching the
predicted quadratic scaling $\lmin\propto\text{parallax}^{2}$ of
\cref{rem:parallax}. \Cref{fig:exp2}(b) shows \emph{which} directions collapse: the
two translation-dominated eigenvalues fall by two orders of magnitude between high- and
low-parallax geometry, while the rotation-dominated eigenvalues are essentially
unchanged---as depth grows the image motion approaches that of a pure rotation and
translation becomes unobservable, exactly as the spectral picture predicts. \Cref{fig:exp2}(c) runs detector (D2) along a trajectory containing a
low-parallax segment; the test $\lmin^{(n)}<\tau_{\lambda}$ fires precisely on the
degenerate frames.

\begin{figure}[t]
\centering
\includegraphics[width=\textwidth]{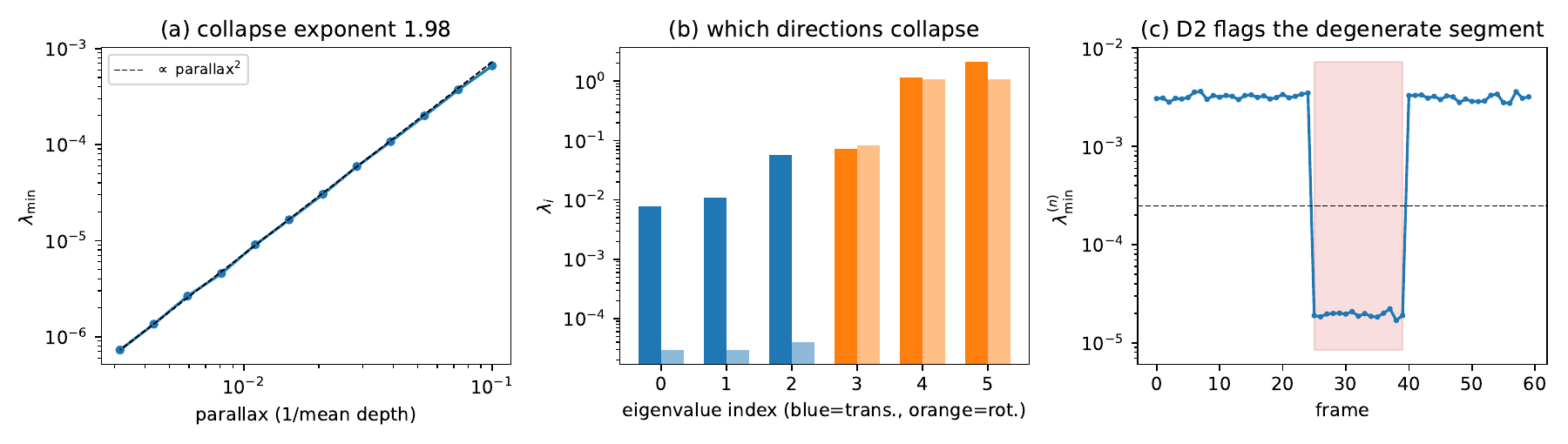}
\caption{Weak observability. (a) $\lmin(\Hoo)$ collapses as parallax$^{2}$ (fitted
exponent $1.98$). (b) Translation eigenvalues collapse under low parallax while
rotation eigenvalues persist. (c) The $\lmin$ test flags the low-parallax segment of a
trajectory.}
\label{fig:exp2}
\end{figure}

\subsection{Fisher equivalence and finite-sample concentration}
\label{sec:exp3}

\Cref{fig:exp3}(a) confirms the Fisher--curvature identity \eqref{eq:fisher} entrywise:
the Monte-Carlo Fisher information $\E[\psi\psi^{\top}]$ and the curvature
$\E[J^{\top}WJ]$ agree to $1.8\times10^{-3}$ relative error. \Cref{fig:exp3}(b) shows
$\opnorm{\Cv_{n}-\Cv}/\opnorm{\Cv}$ decaying with slope $-0.49$
in $n$, the matrix-Bernstein rate underlying \cref{thm:stability}.
\Cref{fig:exp3}(c) propagates this to the diagnostics: the relative errors of $\lmin$
and of \GOI{} decay at the same $O(1/\sqrt n)$ rate, so the tests are reliable already
at moderate feature counts.

\begin{figure}[t]
\centering
\includegraphics[width=\textwidth]{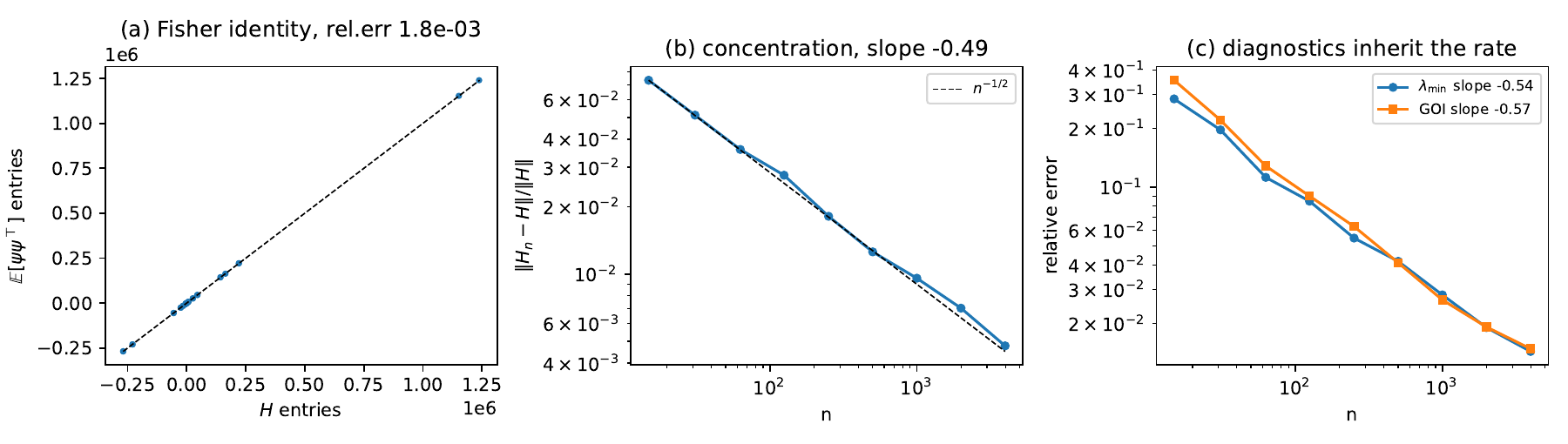}
\caption{Fisher equivalence and concentration. (a) $\E[\psi\psi^\top]$ vs.\ $\Cv$
entrywise (rel.\ error $1.8\times10^{-3}$). (b) Curvature concentration at the
$1/\sqrt n$ rate. (c) The same rate propagates to $\lmin$ and \GOI.}
\label{fig:exp3}
\end{figure}

\subsection{The stability law}
\label{sec:exp4}

\Cref{fig:exp4} tests \cref{thm:stability}. The observable error $\Gnorm{\hat\xi_{n,O}}$
scales as $n^{-0.54}$ at fixed conditioning (a) and as $\lmin^{-0.52}$ as conditioning
is varied through parallax (b), both matching the statistical RMS exponent $-1/2$. In
(c) all runs---three noise levels and three sample sizes---collapse onto a single line
when plotted against $\sigma/\sqrt{n\,\lmin}$, with a $5\%$ coefficient of variation,
confirming the predicted stability law.

The leverage condition (A4) under which \cref{thm:stability} operates is
verified on the same designs, in both its finite-sample and literal forms.
Per scene, the leverage constant
$\max_{z}\sup_{v}v^{\top}(J^{\top}WJ)v/v^{\top}\Cv_{n}v$ lies in
$[4.9,\,13.1]$ (median $8.5$) across every $(n,\text{depth})$ configuration of
panels (a)--(c); it is flat along the conditioning sweep (median $8.2$--$8.8$
while $\lmin$ falls by more than two decades), and $\kappa/n$ decays from
$0.43$ to $0.013$ as $n$ grows from $25$ to $800$. Against a $10^{5}$-point
estimate of the population curvature, the per-feature leverage has median
$2.8$ and maximum below $10$ in every design, again flat in depth. The
experiments therefore sit squarely in the well-spread
regime---$\kappa=O(1)$, independent of the conditioning---in which the
theorem is informative. \Cref{sec:exp5,sec:exp6} report the same measurement
on the real data.

\begin{figure}[t]
\centering
\includegraphics[width=\textwidth]{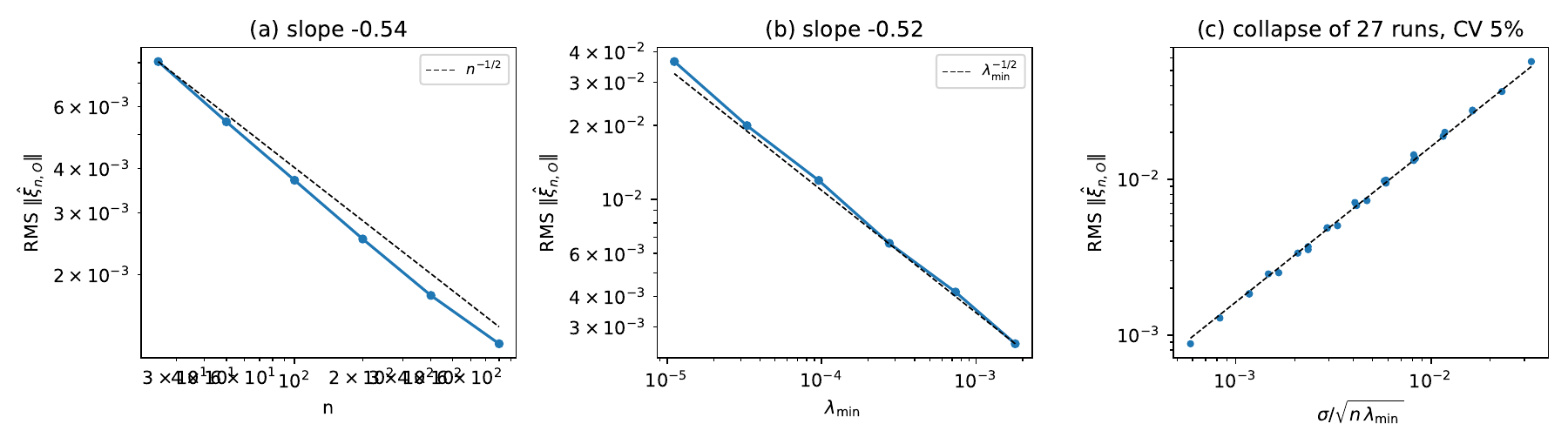}
\caption{Stability radius. (a) Error $\propto n^{-1/2}$. (b) Error
$\propto\lmin^{-1/2}$. (c) Collapse of all $(\sigma,n,\lmin)$ runs onto
$\Gnorm{\hat\xi_{n,O}}\sim\sigma/\sqrt{n\,\lmin}$ (CV $=5\%$).}
\label{fig:exp4}
\end{figure}

\subsection{Real-data study: TUM RGB-D dynamic sequences}
\label{sec:exp5}

We probe the gating question on five TUM RGB-D sequences~\cite{tum}---four
dynamic and one static control---(freiburg3: \texttt{sitting\_static},
\texttt{walking\_xyz},
\texttt{walking\_halfsphere}, \texttt{walking\_rpy}, \texttt{walking\_static}),
using a controlled frame-to-frame PnP protocol that isolates gating quality from
full-SLAM drift: ORB features of frame $t$ are back-projected with the ground-truth
pose of $t$ into map points, the pose of frame $t{+}3$ is estimated by Gauss--Newton,
and all gated methods share an identical machinery---a common Huber-robust
initialization (the influence function is defined at the estimator's solution,
\cref{thm:if}) followed by iterative removal of the worst $q\%$ of features, the
\emph{only} difference being the ranking criterion: \emph{residual gating} ranks
by the whitened residual $r^{\top}Wr$ (the chi-square statistic of
\cref{prop:student}), \emph{score gating} by the raw score norm $\norm{\psi(z)}$,
and \emph{\GOI{} gating} by $\GOI(z)$; unranked Huber IRLS and a RANSAC reference
complete the baselines. The trimming fraction is
selected per method on \texttt{walking\_xyz} and held fixed elsewhere;
\texttt{walking\_xyz} is therefore a tuning sequence, reported for
completeness, and the held-out comparisons are the remaining four sequences.
\Cref{tab:tum} reports translational errors; rotational medians
($0.12$--$0.48^{\circ}$) order identically.

\begin{table}[t]
\centering\small
\caption{TUM RGB-D, controlled protocol: median\,/\,RMSE translational error (m).}
\label{tab:tum}
\begin{tabular}{lccccc}
\toprule
method & sit\_static & walk\_half & walk\_rpy & walk\_static & walk\_xyz\\
\midrule
none            & 0.022/0.042 & 0.177/0.641 & 0.187/0.952 & 0.054/0.196 & 0.168/0.489\\
Huber IRLS      & 0.005/0.008 & 0.020/0.118 & 0.017/0.207 & 0.007/0.014 & 0.016/0.110\\
RANSAC (ref.)   & 0.005/0.008 & 0.016/0.071 & 0.014/0.051 & 0.006/0.010 & 0.013/0.027\\
residual gating & 0.005/0.008 & 0.016/0.133 & 0.014/0.036 & 0.006/0.009 & 0.013/0.071\\
score gating    & 0.005/0.008 & 0.016/0.190 & 0.013/0.040 & 0.006/0.009 & 0.012/0.037\\
\GOI{} gating   & 0.005/0.008 & 0.017/0.103 & 0.013/0.041 & 0.006/0.009 & 0.013/0.027\\
\bottomrule
\end{tabular}
\end{table}

The result is \emph{parity}: per-frame paired comparisons between \GOI{} and
residual gating show win rates of $43$--$58\%$ with Wilcoxon $p>0.1$ on the
principal dynamic sequences, and the apparent RMSE gaps trace to one or two
individual frames rather than a systematic tail effect. \GOI{} gating incurs no
cost on the static control ($p=0.001$ in its favor, but at the $10^{-4}$\,m scale).
This is the outcome the theory predicts: the indoor fr3 geometry is well
conditioned (per-frame $\mathrm{cond}(\Cv)\approx10^{2}$), and with a flat
spectrum $\GOI(z)\propto\norm{\psi(z)}=\norm{J^{\top}Wr}$
(\cref{prop:spectral}), which differs from the whitened residual only through the
per-feature factor $JJ^{\top}$; at the bounded field of view and narrow depth
range of these scenes that factor is nearly uniform across features, so the two
rankings almost coincide. A dose--response
probe---restricting the field of view to raise conditioning---moves
$\mathrm{cond}(\Cv)$ from $98$ to $890$ with no detectable gap opening (win rate
$46$--$55\%$, flat); the synthetic divergence regime
($\mathrm{cond}\approx2.5\times10^{4}$, \cref{sec:exp1}) is not reachable by this
intervention on these scenes. On well-conditioned scenes, then, influence and residual gating coincide, as
\cref{prop:spectral} predicts; the weakly observable regime, where they must
part ways, is probed on KITTI next---where \cref{prop:student} predicts which
statistic prevails when the weak directions are noise-dominated.

\subsection{KITTI: weak observability and the leverage effect}
\label{sec:exp6}

KITTI odometry street scenes sit in the weakly observable regime the theory
targets: with stereo depths of $20$--$80$\,m, the per-frame conditioning
(computed in the camera-anchored chart of \cref{rem:gauge}) is
$\mathrm{cond}(\Cv)\approx10^{4}$---two orders beyond TUM, at the level of the
synthetic divergence regime of \cref{sec:exp1}. Under the identical protocol of
\cref{sec:exp5} on sequences 04 (highway, 135 frame pairs) and 07 (urban, 550
pairs), raw-\GOI{} gating is \emph{significantly worse} than residual gating:
median translational error $0.087$ vs.\ $0.065$\,m and P90 $0.24$ vs.\
$0.14$\,m on 04, and $0.025$ vs.\ $0.023$\,m on 07, with per-frame win rates of
$30\%$ and $42\%$ (Wilcoxon $p<10^{-3}$). The deficit is largest on the more
weakly observable sequence (04); within-sequence stratification by per-frame
conditioning is statistically flat, so the effect resolves at the
sequence-geometry level rather than frame by frame. This is the outcome
\cref{prop:student,rem:leverage} predict when the weak directions are
noise-dominated: the raw index is then governed by leverage---it is expected to
flag far, weak-direction-aligned \emph{inliers}, whose information is scarce and
whose noise (stereo depth error grows as $Z^{2}$) loads exactly those
directions---and trimming them de-informs the weak direction. (The released pipeline directly
supports a per-feature audit of the removed measurements' depth and
weak-direction alignment.) We stress that this mechanism involves an effect outside the measurement
model of \cref{sec:prelim}, which places noise on the image measurement $z$
only (A3): the map points $X$ of the KITTI protocol come from stereo and carry
depth error that grows as $Z^{2}$. The KITTI outcome is therefore consistent
with the prediction of \cref{prop:student} but is not a test of the model in
isolation; propagating map-point uncertainty into an effective measurement
covariance would make the two coincide and is left for future work. With that
caveat, the three testbeds are explained by one mechanism. With
bias-dominated contamination and controlled magnitudes (synthetic,
\cref{sec:exp1}), raw influence separates harm. With a flat spectrum (TUM,
\cref{sec:exp5}), influence and residual coincide. With noise-loaded weak
directions (KITTI), the standardized statistic---the chi-square test---prevails,
as \cref{prop:student} prescribes.

\paragraph{Assumption (A4) on the real data.}
The leverage constant is measured on every frame of both real-data studies,
in the same chart and at the same evaluation points as $\lmin$. Across the
$1{,}296$ TUM frame pairs the per-frame $\kappa$ has median $13$--$15$ and
90th percentile $18$--$33$ per sequence---the same order as the synthetic
designs---and across the $685$ KITTI frame pairs median $35$--$46$ with
$\kappa/n$ at median $0.05$--$0.09$. On all three testbeds $\kappa$ is
uncorrelated with the conditioning (log-correlation $0.00$ on KITTI~04,
$0.25$ on 07): leverage and conditioning are independent axes, which is why
(A4) is a separate hypothesis. The heavy tails have a single, audited cause,
the nearest points in the frame: the top-leverage feature of the worst TUM
frame is the single closest point (depth $0.9$\,m against a frame median of
$2.4$\,m), and the largest KITTI values occur for the \emph{untrimmed} Huber
baseline evaluated at the advanced pose, where forward motion brings
map points near the lower depth cutoff---both consistent with the $1/Z^{2}$
scaling of per-feature translational information (\cref{rem:parallax}). The
gated survivor sets that the estimator actually uses are better spread than
the raw matches (kept-set medians $28$--$31$ on KITTI vs.\ frame-level
$35$--$46$), so \cref{thm:stability} applies to the operative feature sets,
while individual untrimmed frames touching $\kappa/n\approx0.9$ mark the
boundary of its regime---consistent with the sample-size condition of the
theorem.

\section{Discussion and limitations}
\label{sec:discussion}

The synthetic experiments check numerically the identities the theory
states---the influence-function equality (\cref{thm:if}) and the
Fisher--curvature identity (\cref{thm:fisher})---and test the scaling laws it
predicts---the quadratic translational collapse and the stability law
(\cref{thm:stability})---while the real-data studies test the practical
consequences. They are not a benchmark of a complete SLAM system, and
the practical content of the framework is a division of roles, not a better
detector: \GOI{} is the exact measure of a measurement's \emph{effect} on the
estimate (attribution, auditing, sensitivity), while the correctly standardized
statistic for outlier \emph{testing} is the classical chi-square residual
(\cref{prop:student}), whose robustness the framework explains and whose failure
modes it delimits. The metric $G$
enters only through the curvature operator $\Aop=G^{-1}\Cv$ and the whitening isometry
of \cref{lem:projector}; the experiments use $G=I_{6}$, and other choices (e.g.\
block-diagonal weights balancing rotational and translational units) change the index
quantitatively but leave the theory intact. Three limitations follow from this and
from the assumptions. First, with $G=I_{6}$ the spectrum mixes radians and
metres, so $\lmin$, $\mathrm{cond}(\Cv)$, the threshold $\tau_{\lambda}$ of
(D2), and the comparison of translational with rotational eigenvalues in
\cref{fig:exp2}(b) are statements relative to this unit choice (rescaling
translation to millimetres changes $\lmin$ by $10^{-6}$); the theorems are
covariant in $G$, and \cref{thm:fisher} gives the numbers their meaning as
Cram\'er--Rao variances in the chosen units, but ``weakly observable'' is a
metric-relative judgement. Second, \cref{thm:fisher} identifies $\Cv$ with the
Fisher information for the Gaussian score, whereas the real-data estimators are
Huber-initialized; \cref{thm:if,prop:student} apply to any M-estimator score,
but the Cram\'er--Rao reading of \GOI{} is specific to the least-squares score
at $g^{*}$. Third, the model places noise on the image measurement only; the
map-point uncertainty that drives the KITTI result (\cref{sec:exp6}) is
outside (A3). Finally, the name ``observability index'' is retained for
continuity with the diagnostics of \cref{sec:diagnostics}; the quantity itself is
a per-measurement influence norm---the $\SE$ counterpart of Cook's
distance---whose connection to observability is through $\lmin$ in
\cref{prop:spectral}. The real-data studies of \cref{sec:exp5,sec:exp6} close the loop: parity on
well-conditioned scenes, leverage-driven degradation of raw-influence gating
under weak observability, and the chi-square test recovered as the standardized
influence, with the model-mismatch caveat of \cref{sec:exp6}. A natural next step is the integration of the diagnostics into a
complete SLAM back end.
The framework applies unchanged to two-view essential-matrix refinement following
an eight-point initialization~\cite{hz}: taking $r$ to be the Sampson error makes
the refinement a Gauss--Newton M-estimator to which
\cref{thm:if,thm:fisher,prop:student} apply verbatim, and the classical
pure-rotation and planar degeneracies then appear as rank deficiencies of $\Cv$,
i.e.\ as collapse of $\lmin$; a quantitative two-view study is left to future
work. The framework assumes a calibrated central-projection model and
left-trivialized perturbations; extension to rolling-shutter and multi-camera
rigs requires re-deriving $J$ but leaves the operator theory intact.

\section{Conclusion}
\label{sec:conclusion}

We presented the Geometric Observability Index, a single per-feature quantity---the
$G$-norm of a measurement's influence on the estimated pose---that unifies influence
functions, Fisher information, conditioning, and stability through the spectral
geometry of the curvature operator on $\SE$, and that, once standardized by its
inlier null covariance, collapses exactly to the classical chi-square residual test.
Each identity is checked numerically and each practical prediction is tested on
real TUM and KITTI sequences; in particular, the framework explains why residual
gating behaves as it does---robust to residual-visible contamination, blind to
high-leverage masking---and where raw influence misleads. We hope the
released implementation makes the attribution and degeneracy diagnostics easy to
adopt in existing Gauss--Newton estimators.

\section*{Reproducibility}
\label{sec:repro}
\addcontentsline{toc}{section}{Reproducibility}
All experiments are produced by a self-contained Python implementation, released at
\url{https://anonymous.4open.science/r/goi-se3}\footnote{Anonymized repository during concurrent peer review; a public repository will be linked in a future version.}: \texttt{code/} contains the core library
and the synthetic experiments (NumPy/SciPy/Matplotlib only; \texttt{run\_all.sh}
regenerates every figure and number of
\cref{sec:exp1,sec:exp2,sec:exp3,sec:exp4} in about three minutes), and
\texttt{tum/} and \texttt{kitti/} contain the real-data pipelines of
\cref{sec:exp5,sec:exp6}. The analytic Jacobian and the Fisher--curvature identity
are checked by an included test (\texttt{tests.py}); random seeds are fixed
throughout, so all reported numbers reproduce bit-for-bit.

\section*{Acknowledgment of AI usage}
\addcontentsline{toc}{section}{Acknowledgment of AI usage}
An AI assistant was used to help check mathematical consistency and edit prose.
All theoretical claims and proofs, the implementation, the numerical validation,
and the final text were verified by the authors, who take full responsibility
for the content.

\end{document}